\documentclass[sigconf]{acmart}
\makeatletter
\def\@ACM@checkaffil{
    \if@ACM@instpresent\else
    \ClassWarningNoLine{\@classname}{No institution present for an affiliation}%
    \fi
    \if@ACM@citypresent\else
    \ClassWarningNoLine{\@classname}{No city present for an affiliation}%
    \fi
    \if@ACM@countrypresent\else
        \ClassWarningNoLine{\@classname}{No country present for an affiliation}%
    \fi
}
\makeatother
\usepackage{color}
\usepackage{float}
\usepackage{graphicx}
\usepackage{amsfonts}
\usepackage{latexsym,amsmath}
\usepackage{todonotes}
\usepackage[ruled]{algorithm}
\usepackage[noend]{algpseudocode}
\algdef{SE}[DOWHILE]{Do}{doWhile}{\algorithmicdo}[1]{\algorithmicwhile\ #1}
\usepackage{comment}
\usepackage{epstopdf}
\usepackage{soul}
\usepackage{cleveref}

\title{Capturing a Moving Target by Two Robots in the F2F Model}
\author{Khaled Jawhar}
\affiliation{%
  \institution{School of Computer Science, Carleton University}
  \city{Ottawa}
  \state{ON}
  \country{Canada}
}
\author{Evangelos Kranakis}
\affiliation{%
  \institution{School of Computer Science, Carleton University}
  \city{Ottawa}
  \state{ON}
  \country{Canada}
}
\begin{document}
\begin{abstract}
We study a search problem on capturing a moving target on an infinite real line. Two autonomous mobile robots (which can move with a maximum speed of 1) are initially placed at the origin, while an oblivious moving target is initially placed at a distance 
$d$ away from the origin. The robots can move along the line in any direction, but the target is oblivious, cannot change direction, and moves either away from or toward the origin at a constant speed 
$v$. Our aim is to design efficient algorithms for the two robots to capture the target. The target is captured only when both robots are co-located with it. The robots communicate with each other only face-to-face (F2F), meaning they can exchange information only when co-located, while the target remains oblivious and has no communication capabilities.

We design algorithms under various knowledge scenarios, which take into account the prior knowledge the robots have about the starting distance 
$d$, the direction of movement (either toward or away from the origin), and the speed $v$ of the target. As a measure of the efficiency of the algorithms, we use the competitive ratio, which is the ratio of the capture time of an algorithm with limited knowledge to the capture time in the full-knowledge model.

In our analysis, we are mindful of the cost of changing direction of movement, and show how to accomplish the capture of the target with at most three direction changes (turns).

{\bf Key words and phrases.} Autonomous robot, Capture time, Competitive ratio, F2F (Face-to-Face), Knowledge model, Oblivious target, Searcher, Speed, Turn.
\end{abstract}
\maketitle
\section{Introduction}

We study the problem of linear search for an oblivious moving target by two autonomous mobile robots. Linear search problems have been extensively studied and applied in various domains such as data mining, surveillance, and rendezvous, cf.\cite{gal_search_games,ahlswede1987search}. Evacuation (also known as group search) is a related problem in which multiple robots cooperate to find an unknown target, cf.\cite{group_search}. Capturing a moving target can be viewed as a form of group search where the target itself is mobile. This problem has been explored in numerous settings, including on graphs, cf.\cite{anthony2011game}, as well as in scenarios involving chasing and escaping, cf.\cite{chases_escapes}. A special case of this problem is evacuation with a stationary target ($v=0$).

Several foundational works, including~\cite{BCR93,beck1964linear,bellman1963optimal}, have initiated the study of linear search and evacuation with one or more robots moving at uniform or varying speeds. The focus of this research is to devise algorithms for cooperative searchers that achieve optimal or near-optimal upper and lower bounds by analyzing the competitive ratio.

The competitive ratio provides a measure of the efficiency of an algorithm by comparing the performance of a given algorithm under limited knowledge constraints to the performance in a full-knowledge model. In the full-knowledge model, robots are aware of all input parameters, such as the starting location of the target, its speed, its direction of movement, and its distance from the origin. The optimal algorithm is the one that minimizes the competitive ratio for linear search time, evacuation time, or capture time, depending on the specific problem under consideration.

The primary motivation of our study on capturing a moving target by two robots in the F2F communication model is to understand the impact of input knowledge constraints on capture time. These constraints include the knowledge of the speed and initial distance of the target from the origin. Furthermore, our approach treats a searcher's "turn in direction" as a resource. We propose algorithms that ensure the target is captured with a constant and minimal number of turns.

\subsection{Preliminaries and Notation}

In this section, we define the basic concepts of robot mobility and communication, as well as the four basic knowledge models that we will analyze in the sequel.

The search domain is the infinite line and it is bidirectional in that the robots can move in either direction without this affecting their speed. The searchers can move with maximum speed $1$. The mobile target is oblivious in that it is unable to change its speed, its direction of movement, and has no communication capabilities; it can move either away from or toward the origin with maximum speed $v$, and this is fixed as part of the input. If it is moving away from the origin, we assume that $v < 1$, and we call this the {\em away} model. If the target is moving toward the origin, we assume that $v$ is arbitrary, and we call this the {\em toward} model. It will be convenient throughout the paper to identify the two searching autonomous robots as $R_1$ and $R_2$; however, their capabilities are indistinguishable.

An algorithm capturing the target is a complete description of the trajectories traced by the two robots until they are both co-located with the mobile target (which is moving at speed $v$). Note that only one of the two searchers is not sufficient to capture the moving target; instead, the capture of the target is considered complete only when both robots are co-located with the target. Clearly, if $v = 0$, then the target is static, in which case capturing the target is equivalent to standard evacuation. Our strategies take into account that there is a cost for changing direction. In our algorithms, we will show that the searchers capture the moving target while the number of turns (changes in direction) by either searcher remains constant (at most 3).

The competitive ratio of a strategy (or algorithm) $\mathcal{S}$, denoted by $CR_{\mathcal{S}}$, is defined as the supremum over all starting positions of the target of the time the agents take to capture the target divided by the time it would take if the target's movement was fully known (initial position, speed, direction of movement) to the searchers. The competitive ratio of a certain problem is the infimum of $CR_{\mathcal{S}}$ over all possible strategies $\mathcal{S}$ for the problem. Throughout this paper, we will use the abbreviation $CR$ to refer to the term "competitive ratio". The efficiency or optimality of an algorithm is measured using the competitive ratio.

The searchers are autonomous mobile agents that can move around with maximum speed $1$. In this paper, we are considering the FaceToFace model and refer to it by F2F. In this model, both robots can exchange information only if they are co-located (i.e., meet at a specific point on the line). A typical communication exchange may involve information such as 'target is found', 'move in this direction', etc. Both robots are endowed with pedometers and have computing abilities to deduce the location of each other, through relevant communication exchanges. The robot performs its movement based on calling atomic operations such as {turn, stop, increase/decrease speed, etc.}. As an example, if the robot is moving and it needs to turn, it will use the following operations: {stop, turn, restart, accelerate}. In our analysis, we also take into account the number of turns (changes of direction). Therefore, the goal is to find an algorithm with the best competitive ratio that reduces the number of turns. Alternatively, in some of our algorithms, we increase the speed instead, which is less costly than making a turn; this is because we only need to call accelerate instead of calling turn, while the robot is moving, which requires four atomic operations.

We consider algorithms under various constraints reflecting the knowledge the two robots have about the initial distance of the target from the origin, its direction of movement (away or toward), and the speed of the moving target. We will study algorithms under different knowledge models. The FullKnowledge (FK) model refers to two mobile agents that have complete knowledge of the initial distance, direction of movement, and the speed of the moving target. The NoSpeed model refers to no knowledge of the speed of the target. The NoDistance (ND) model refers to no knowledge of the initial distance of the target. The NoKnowledge (NK) model refers to the case when the robots have no knowledge of the speed or the initial distance of the target.
 
\subsection{Related Work}

Linear search refers to searching on a line for a target located on the real line. The study of linear search problems for a single robot was first initiated by Beck~\cite{beck1964linear} and Bellman~\cite{bellman1963optimal}. Assuming that the distance and the direction to the target are unknown, they proposed an optimal algorithm with a competitive ratio of $9$. Additional work, including deterministic variants of linear search, can be found in~\cite{BCR93}.

There are numerous variants of search and evacuation problems. These can involve environments with multiple robots operating at distinct speeds~\cite{bampas2019linear,isaacCzyzowiczGKKNOS16,PODC16}, as well as different search domains such as disks, triangles, and circles~\cite{brandt2017collaboration,chuangpishit2018evacuating,czyzowicz2018evacuating}.

Further extensions to these problems explore augmenting the capabilities of the robots with additional agents. For instance, the inclusion of a bike as an auxiliary mobile agent to assist robots in reaching the target faster was studied in~\cite{jawhar2021robot}, where robot evacuation on a line was considered in the presence of such an agent.

In the context of this paper, we focus on algorithms for capturing a moving target. This problem resembles the evacuation of robots toward a static target, but with the added complexity of a moving target. The case of a moving target was first studied by McCabe~\cite{mccabe1974}, who investigated searching for an oblivious target following a Bernoulli random walk on the integers. A deterministic oblivious target was first considered in Alpern and Gal~\cite{gal_search_games}[Section 8.5], where the target moves away from the origin at a constant speed $v < 1$, which is known to the searching robot. In their work, the initial distance of the target is unknown, and they provide an algorithm with an optimal competitive ratio for this scenario.

Our work is strongly influenced by and builds upon the results of~\cite{coleman2022line}, where the authors analyzed the competitive ratio under four knowledge models: FullKnowledge, NoSpeed, NoDistance, and NoKnowledgeOfSpeedAndDistance. While their analysis focuses on a single searcher, we extend their work to two searchers operating in the F2F communication model. Additionally, the tightness of the main lower bound for the NoSpeed model with a single searcher was recently established in~\cite{coleman2024lineiwoca}.

The cost of turns in a search algorithm was first proposed in~\cite{gal_search_games} and further investigated in~\cite{angelopoulos2017infinite,demaine2006online}. However, the strategies in~\cite{coleman2022line} do not account for the cost of changing direction. In contrast, our study focuses on capturing a moving target with two robots under the F2F communication model, explicitly considering turn costs. We provide upper and lower bounds for the capture problem while maintaining a constant number of turns and analyzing the impact of various knowledge constraints on the competitive ratio.

\subsection{Outline and Results of the Paper}

In this section, we outline the main results of the paper based on the speed $v$ of the target.

\subsubsection{Full Knowledge Model}

Section~\ref{sec:FKM} discusses the Full Knowledge Model. Results are similar to those in~\cite{coleman2022line} and are summarized in Table~\ref{table1}.

\begin{table}[H]
\caption{Competitive Ratios (CRs) in the Full Knowledge Model.}
\begin{center}
\begin{tabular}{|c| c | c| c|}
\hline
Knowledge & Algorithm & Movement & Competitive Ratio \\
\hline
$v,d$ & Algorithm~\ref{FKGoTogetherAwayAlgorithm}& Away & $\frac{3-v}{1-v}$ if $v<1$ \\
\hline
$v,d$ & Algorithm~\ref{FKGoTogetherTowardAlgorithm}  & Toward & $\frac{3+v}{1+v}$ if $v<1$ \\
\hline
$v,d$ & Algorithm~\ref{FKStayAtOriginTowardAlgorithm} & Toward & $\frac{v+1}{v}$ if $v>1$ \\
\hline
\end{tabular}
\end{center}
\label{table1}
\end{table}

\subsubsection{No Distance Model}

Section~\ref{sec:NDM} is about the No Distance Model, and results are summarized in Table~\ref{table2}.

\begin{table}[H]
\caption{CRs in the No Distance Model.}
\begin{center}
\begin{tabular}{|c| c | c| c | c|}
\hline
Knowledge & Algorithm & Movement & Competitive Ratio\\
\hline
$v$ & Algorithm~\ref{NDAwayZigZagTillOriginAlgorithm}  & Away & $\frac{(v+3)^2}{(1-v)^2}$ \\
\hline
$v$ & Algorithm~\ref{NDAwayMovingInOppDirectionAlgorithm}  & Away & $\frac{(v+3)^2}{(1-v)^2}$ \\
\hline
$v$ & Algorithm~\ref{NDTowardZigZagTillOriginAlgorithm}  & Toward & $(\frac{v-3}{v+1})^2$ if $v\geq 3$ \\
\hline
$v$ & Algorithm~\ref{alg:tf2f1}  & Toward & $(\frac{v-3}{v+1})^2$ if $v\geq 3$ \\
\hline
$v$ & Algorithm~\ref{FKStayAtOriginTowardAlgorithm} & Toward & $1+\frac{1}{v}$ if $v\leq 3$ \\
\hline
\end{tabular}
\end{center}
\label{table2}
\end{table}

Note that in Table~\ref{table2}, Algorithms~\ref{NDAwayMovingInOppDirectionAlgorithm} and~\ref{alg:tf2f1} require at most 3 changes of direction. They have the same competitive ratios as Algorithms~\ref{NDAwayZigZagTillOriginAlgorithm} and~\ref{NDTowardZigZagTillOriginAlgorithm}, respectively, but for the last two algorithms, the number of turns in direction is unbounded.  

\subsubsection{No Speed Model}

Section~\ref{sec:NSM} is about the No Speed Model, and results are summarized in Table~\ref{table3}.

\begin{table}[H]
\caption{CRs for the No Speed Model:$M= \left(\frac 1{1-v}\right)$.}
\begin{center}
\begin{tabular}{|c| c | c| c | c|}
\hline
Knowledge & Algorithm & Movement & Competitive Ratio\\
\hline
$d$ & Algorithm~\ref{NSAwayAlgorithm}  & Away  & 
 $O\left( M^{8} (\log^2 M)/d\right)$\\ 
\hline
$d$ & Algorithm~\ref{NSTowardAlgorithm}  & Toward  & $3$ \\
\hline
\end{tabular}
\end{center}
\label{table3}
\end{table}

\subsubsection{No Knowledge Model}

Section~\ref{sec:NKM} discusses the No Knowledge Model, and results are summarized in Table~\ref{table4}.

\begin{table}[H]
\caption{CRs for the No Knowledge Model: $M= \max \left(d, \frac 1{1-v}\right)$.}
\begin{center}
\begin{tabular}{|c| c | c| c | c|}
\hline
Knowledge & Algorithm & Movement & Competitive Ratio\\
\hline
None & Algorithm~\ref{NKAwayAlgorithm}  & Away & $O\left( M^{10} (\log^2 M) / d \right)$\\
\hline
None & Algorithm~\ref{FKStayAtOriginTowardAlgorithm}  & Toward & $1+\frac{1}{v}$\\
\hline
\end{tabular}
\end{center}
\label{table4}
\end{table}

Note that the number of turns (changes of direction) in all algorithms is at most 3, with the exception of Algorithms~\ref{NDAwayZigZagTillOriginAlgorithm} and~\ref{NDTowardZigZagTillOriginAlgorithm}.

\section{Full Knowledge Model}
\label{sec:FKM}

We begin with the analysis of the FullKnowledge model and distinguish the cases where the target is moving away or toward the origin. Note that Algorithms~\ref{FKGoTogetherAwayAlgorithm},~\ref{FKGoTogetherTowardAlgorithm},~\ref{FKStayAtOriginTowardAlgorithm} below are similar to the corresponding single robot algorithms in~\cite{coleman2022line} the only difference being that the two robots stay together as a single robot. We include the proofs for completeness. Alternatives to Algorithms~\ref{FKGoTogetherAwayAlgorithm},~\ref{FKGoTogetherTowardAlgorithm} arise when the searchers move separately in opposite directions but it is easily verified that this will not change the CRs.  The lower bound proofs are slightly different 

\subsection{Target moving away from the origin}

Here is the outline of the algorithm: Both robots know the speed and the distance to the target, thus they both move in the same direction distance $\frac{d}{1-v}$. If they don't find the target, then they switch direction till they reach the target. 
\begin{algorithm}[H]
\caption{FKGoTogetherAway ($S$ source, $D$ destination)}
\label{FKGoTogetherAwayAlgorithm}
\begin{algorithmic}[1]
\State {$R_{1}$ and $R_{2}$ move together in any direction with speed $1$}
\If {$R_{1}$ and $R_{2}$ don't reach the target at point $\frac{d}{1-v}$}
\State {Both robots switch direction and move till they reach the target}
\EndIf
\end{algorithmic}
\end{algorithm}
We can prove the following theorem.
\begin{theorem}
\label{FKGoTogetherAwayTheoremUpperBound}
For the full knowledge away model,
the competitive ratio of Algorithm~\ref{FKGoTogetherAwayAlgorithm} is at most $\frac{3-v}{1-v}$ 
\end{theorem}
\begin{proof}(Theorem~\ref{FKGoTogetherAwayTheoremUpperBound})
Robots $R_{1}$ and $R_{2}$ move in the same direction for a distance of $\frac{d}{1-v}$. If neither of them catches the target during this time, then the target will be ahead of the robots by a distance of $\frac{2d}{1-v}$. To cover this remaining distance, they need an additional time of $\frac{\frac{2d}{1-v}}{1-v}$.

Therefore, the total time required for both robots to capture the target is:
\[
\frac{d}{1-v} + \frac{\frac{2d}{1-v}}{1-v} = \frac{3d - dv}{(1-v)^2}.
\]
This yields the competitive ratio:
\[
CR = \frac{\frac{3d - dv}{(1-v)^2}}{\frac{d}{1-v}} = \frac{3-v}{1-v}.
\]
This proves Theorem~\ref{FKGoTogetherAwayTheoremUpperBound}.
\end{proof}
In the next theorem Algorithm~\ref{FKGoTogetherAwayAlgorithm} is shown to be optimal.
 \begin{theorem}
\label{FullKnowledgeAwayTheoremLowerBound}
The competitive ratio of any algorithm in the FullKnowledgeAway model is at least $\frac{3-v}{1-v}$ 
\end{theorem}
\begin{proof}(Theorem~\ref{FullKnowledgeAwayTheoremLowerBound})
If both robots $R_{1}$ and $R_{2}$ are at the origin and the moving target is at a distance $d$ away from the origin, then any of the two robots needs $\frac{d}{1-v}$ to catch up with the target. Assume $R_{1}$ reaches point $\frac{d}{1-v}$, then the adversary would place the target at $-\frac{d}{1-v}$. At this point, $R_{2}$ catches the target and $R_{1}$ would be far from the target by $\frac{2d}{1-v}$. Thus, for $R_{1}$ to reach the target, it needs time $\frac{\frac{2d}{1-v}}{1-v} = \frac{2d}{(1-v)^2}$. The competitive ratio is then:
\[
CR = \frac{\frac{2d}{(1-v)^2} + \frac{d}{1-v}}{\frac{d}{1-v}} = \frac{3-v}{1-v}.
\]
This proves Theorem~\ref{FullKnowledgeAwayTheoremLowerBound}.
\end{proof}
\subsection{Target moving toward the origin}
 The target moves with speed $v$ toward the origin. There are two cases to consider:

\noindent{\em Case 1: $v\leq 1$} 

Here is the outline of the algorithm: Since both robots know the speed and the distance to the target, they move in the same direction distance $\frac{d}{1+v}$. If they don't catch the target, then they switch direction and keep moving till they reach the target. The algorithm will be as follows:

\begin{algorithm}[H]
\caption{FKGoTogetherToward ($S$ source, $D$ destination)}\label{FKGoTogetherTowardAlgorithm}
\begin{algorithmic}[1]
\State {$R_{1}$ and $R_{2}$ move together in any direction with speed $1$}
\If {$R_{1}$ and $R_{2}$ don't reach the target at point $\frac{d}{1+v}$}
\State {Both robots switch direction and move till they reach the target}
\EndIf
\end{algorithmic}
\end{algorithm}

We prove the following theorem.
\begin{theorem}
\label{FKGoTogetherTowardTheorem}
For the full knowledge toward model, the competitive ratio of Algorithm~\ref{FKGoTogetherTowardAlgorithm} is at most $\frac{3+v}{1+v}$ if $v\leq1$.
\end{theorem}
\begin{proof}(Theorem~\ref{FKGoTogetherTowardTheorem})
Both robots $R_{1}$ and $R_{2}$ move in opposite directions for a distance of $\frac{d}{v+1}$. If $R_{1}$ reaches the target, then $R_{2}$ would be at a distance $\frac{2d}{1+v}$ in the other direction. Thus, it needs time $\frac{2d}{(v+1)^2}$ to reach the target. The total time needed for both robots to evacuate will be:
\[
\frac{d}{v+1} + \frac{2d}{(v+1)^2} = \frac{3d+dv}{(v+1)^2}.
\]
The competitive ratio is then:
\[
CR = \frac{\frac{3d+dv}{(v+1)^2}}{\frac{d}{v+1}} = \frac{3+v}{1+v}.
\]
This proves Theorem~\ref{FKGoTogetherTowardTheorem}.
\end{proof}
\noindent{\em Case 2: $v>1$} 

Here is the outline of the algorithm: Both robots stay at the origin waiting for the target to reach the origin. The algorithm will be as follows:

\begin{algorithm}[H]
\caption{FKStayAtOriginToward ($S$ source, $D$ destination)}\label{FKStayAtOriginTowardAlgorithm}
\begin{algorithmic}[1]
\State {$R_{1}$ and $R_{2}$ stay at the origin.} 
\State {The target moves toward the origin with speed $v$ and meets the two robots}
\end{algorithmic}
\end{algorithm}
We prove the following results.

\begin{theorem}
\label{FKStayAtOriginTowardTheorem}
For the full knowledge toward model the competitive ratio of Algorithm~\ref{FKStayAtOriginTowardAlgorithm} is at most $\frac{v+1}{v}$ if $v \geq 1$.
\end{theorem}
\begin{proof}(Theorem~\ref{FKStayAtOriginTowardTheorem})
If both robots wait for the moving target at the origin, then the capture time will be $\frac{d}{v}$, thus the competitive ratio would be:
\[
CR = \frac{\frac{d}{v}}{\frac{d}{v+1}} = \frac{v+1}{v}.
\]
This proves Theorem~\ref{FKStayAtOriginTowardTheorem}.
\end{proof}
\begin{theorem}
\label{FullKnowledgeTowardLowerBoundTheorem}
The competitive ratio of any algorithm in the FullKnowledgeToward model is at least $\frac{v+1}{v}$ if $v>1$ and $\frac{3+v}{1+v}$ if $v<1$.
\end{theorem}
\begin{proof}(Theorem~\ref{FullKnowledgeTowardLowerBoundTheorem})
Consider point $a$ in any direction away from the origin. It takes time $\frac{d-a}{v}$ for the target to reach $a$. On the other hand, it takes time $\frac{a}{1+v}$ for any of the two robots to reach point $a$. If one of the two robots reaches point $a$, then in the worst-case scenario, the adversary would place the target at the other side, and in this case, it would take the other robot time $\frac{2a}{1+v}$ to reach the target. Thus, the competitive ratio in this case would be as follows:
\[
CR = \frac{\frac{d-a}{v} + \frac{2a}{1+v}}{\frac{d}{1+v}} = \frac{d + dv - a + av}{dv} = 1 + \frac{1}{v} + \frac{a(v-1)}{dv}.
\]
Thus, based on this equation and the value of $a$, we have two cases to consider:
\begin{itemize}
    \item {\em Case~1: If $v \geq 1$}, then clearly $CR \geq 1 + \frac{1}{v}$.
    \item {\em Case~2: If $v \leq 1$}, then $a \geq \frac{d}{1+v}$, otherwise the robots would have captured the target. The competitive ratio satisfies:
    \[
    CR \geq 1 + \frac{1}{v} + \frac{\frac{d(v-1)}{1+v}}{dv} = 1 + \frac{2}{1+v} = \frac{3+v}{1+v}.
    \]
\end{itemize}

This proves Theorem~\ref{FullKnowledgeTowardLowerBoundTheorem}.
\end{proof}

\section{No Distance Model}
\label{sec:NDM}

In this section, we consider the NoDistance model and distinguish the cases where the target is moving away or toward the origin. There are two different classes of algorithms that have identical competitive ratios: ZigZag and NonZigZag.

In the case of ZigZag algorithms, there are two possibilities. The first algorithm is inspired by the single searcher algorithm in~\cite{gal_search_games}, where the two robots move together as a single searcher (we will not include the proof of this here). In the second algorithm (which we describe below), the robots move separately and both follow a ZigZag strategy in that they come back to the origin after every iteration. This means that each robot will use the ZigZag strategy to cover one side of the line starting from the origin.

The second class of algorithms is NonZigZag. The two robots move in opposite directions with a certain optimal speed $u$ (to be determined in the course of the proof). When any of the robots finds the target, it switches its direction and goes back with its maximum unit speed to inform the other robot. Then both robots proceed to the target with their maximum unit speed (see~\cite{chrobak2015group}). What is surprising is that all these algorithms are shown to have the same competitive ratio. However, the NonZigZag algorithm is unique to the two-robot search and is superior in that it minimizes the total number of turns (changes of direction).

Note that the observations above apply to the away as well as the toward case of the two-robot search considered in our paper (see also~\cite{coleman2022line}).

\subsection{Target moving away from the origin}
\label{sec:Target moving away from the origin}

First, we consider the ZigZag case. In the first algorithm, the two robots move together as a single searcher executing a ZigZag search. It is shown in~\cite{gal_search_games} that this algorithm has a competitive ratio of $\frac{(v+3)^2}{(1-v)^2}$.

Next, we design an algorithm in which each of the two robots uses a separate ZigZag strategy only on one side of the origin.

\begin{algorithm}[H]
\caption{NDAwayZigZagTillOrigin}\label{NDAwayZigZagTillOriginAlgorithm}
\begin{algorithmic}[1]
 \For{$k \gets 1$ to $\infty$} 
  \State{$R_{1}$ moves in one direction a distance $a^k$ unless the target is found, then comes back to the origin} 
  \State{$R_{2}$ moves in the other direction a distance $a^k$ unless the target is found, then comes back to the origin}
 \If{Target is found by $R_{1}$}
     \State{Switch direction and move to catch $R_{2}$};
     \State{Both Robots move back to catch up with the target}
     \State{Quit;}
 \Else
 \If{Target is found by $R_{2}$}
    \State{Switch direction and move to catch $R_{1}$};
     \State{Both Robots move back to catch up with the target}
     \State{Quit;}
   \EndIf
    \EndIf
\EndFor
\end{algorithmic}
\end{algorithm}

We prove the following result.

\begin{theorem}
\label{NDAwayZigZagTillOriginTheorem}
Setting $a=\frac{2(1+v)}{1-v}$, the competitive ratio of Algorithm~\ref{NDAwayZigZagTillOriginAlgorithm} is at most $\frac{(v+3)^2}{(1-v)^2}$.
\end{theorem}
\begin{proof} (Theorem~\ref{NDAwayZigZagTillOriginTheorem})
If the target is at distance $d$ from the origin, then eventually either $R_1$ or $R_2$ will capture the target. Without loss of generality, assume that $R_1$ captures the target. The worst-case scenario occurs if, during iteration $k-1$, $R_1$ just misses the target. Specifically, when $R_1$ reaches $x_{k-1}$, the target is at distance $x_{k-1} + \epsilon$. Thus, for $R_1$ to capture the target, it requires time:
\begin{equation} \label{equation1} 2 \sum_{i=0}^{k-1} x_i + \frac{d + 2v \left(\sum_{i=0}^{k-1} x_i\right)}{1-v}. \end{equation}

At this point, $R_2$ is at a distance of
$\frac{d+2v(\sum_{i=0}^{k-1} x_{i})}{1-v}$ on the opposite side of $R_1$. For $R_1$ to meet $R_2$, it requires time:
\begin{equation} \label{equation2} x_k - \frac{d + 2v \left(\sum_{i=0}^{k-1} x_i\right)}{1-v} + \frac{d + 2v \left(\sum_{i=0}^{k-1} x_i\right)}{1-v} = x_k. \end{equation}

At this time, the target is at a distance of $x_k + vx_k$ from both robots. Thus, for the robots to capture the target, they require additional time:
\begin{equation} \label{equation3} \frac{x_k + vx_k}{1-v}. \end{equation}

Using equations \eqref{equation1}, \eqref{equation2}, and \eqref{equation3}, the competitive ratio is given by:
\begin{align} CR &= \frac{\frac{2 \sum_{i=0}^{k-1} x_i - 2v \sum_{i=0}^{k-1} x_i + d + 2v \sum_{i=0}^{k-1} x_i + x_k - vx_k + x_k + vx_k}{1-v}}{\frac{d}{1-v}} \nonumber \\
&= \frac{2x_k + d + 2 \sum_{i=0}^{k-1} x_i}{d} \nonumber \\
&= \frac{2 \sum_{i=0}^{k-1} x_i + 2x_k}{d} + 1.
\end{align}

To analyze further, assume that in round $k$, the robot captures the target. Then, in round $k-1$, the following must hold:
\begin{align} &\frac{d + \left(2a^0 + \dots + 2a^{k-2}\right)v}{1-v} \geq a^{k-1}, \nonumber \\
&\implies d > a^{k-1} - va^{k-1} - 2v \frac{a^{k-1} - 1}{a-1}, \nonumber \\
&\implies d > \frac{a^k - a^{k-1} - va^k + va^{k-1} - 2va^{k-1} + 2v}{a-1}, \nonumber \\
&\implies d > \frac{a^k - a^{k-1} - va^k - va^{k-1} + 2v}{a-1}. \label{equation4}
\end{align}

Thus, the competitive ratio becomes:
\begin{align} CR &= \frac{2ax_k - 2}{a^k - a^{k-1} - va^k - va^{k-1} + 2v} + 1 \
&= \frac{2a^2}{a - 1 - av - v} + 1.
\end{align}

Let us find the optimal value of $a$ that minimizes the competitive ratio. Define $f(a) = \frac{2a^2}{a - 1 - av - v} + 1$. Its derivative $f'(a)$ is given by:
\begin{align} f'(a) &= \frac{4a(a - 1 - av - v) - 2a^2(1 - v)}{(a - 1 - av - v)^2} \nonumber \\
&= \frac{4a^2 - 4a - 4va^2 - 4av - 2a^2(1-v)}{(a - 1 - av - v)^2} \nonumber \\
&= \frac{2a^2 - 2va^2 - 4a - 4av}{(a - 1 - av - v)^2}.
\end{align}

Setting $f'(a) = 0$ gives $a = \frac{2(1+v)}{1-v}$. Substituting this value of $a$ into the competitive ratio expression yields:
\begin{align} CR &= 1 + \frac{8(1+v)^2}{(1-v)^2 \cdot (a - 1 - av - v)} \nonumber \\
&= 1 + \frac{8(1+v)^2}{(1-v)^2 \cdot \left(\frac{2+2v-1+v-2v(1+v)-v+v^2}{1-v}\right)} \nonumber \\
&= \frac{8(1+v)^2}{(1-v)(1-v^2)} \nonumber \\
&= \frac{8(1+v)}{(1-v)^2} + 1 \nonumber \\
&= \frac{(v+3)^2}{(1-v)^2}.
\end{align}

This completes the proof of Theorem~\ref{NDAwayZigZagTillOriginTheorem}.
\end{proof}

\begin{theorem}
\label{NDAwayZigZagTillOrigin_NumberOfTurnsTheorem}
The number of turns for Algorithm~\ref{NDAwayZigZagTillOriginAlgorithm} is at most 
\[
1 + 2\log\left(\frac{2d}{1-v}\right).
\]
\end{theorem}
\begin{proof}(Theorem~\ref{NDAwayZigZagTillOrigin_NumberOfTurnsTheorem})
Let us assume that the target was detected by one of the robots at iteration $k$. Then at the end of iteration $k-1$, both robots would have completed $2(k-1)$ turns. At iteration $k$, one of the robots would make a turn to meet the other robot at the origin, since the other robot already caught the target. Thus, the total number of turns would be $2(k-1)+1=2k-1$. 

In order to find the number of turns in terms of $d$ and $v$, we assume that the target is caught by one of the robots at iteration $k$. Then, at iteration $k-1$, and based on Equation~\ref{equation4}, we have:
\[
\frac{a^{k-1}(a-1-av-v)}{a} \leq d,
\]
which implies the following:
\begin{align*}
    a^{k-1} &\leq \frac{ad}{a-1-av-v},\\
    &\leq \frac{\frac{2d(1+v)}{1-v}}{\frac{2+2v-1+v-2v-2v^2-v+v^2}{1-v}},\\
    &= \frac{2d(1+v)}{1-v^2},\\
    &= \frac{2d}{1-v}.
\end{align*}
Taking the logarithm, we obtain:
\[
k \leq 1 + \log\left(\frac{2d}{1-v}\right).
\]
Thus, the total number of turns would be at most:
\[
2k-1 = 1 + 2\log\left(\frac{2d}{1-v}\right).
\]
This proves Theorem~\ref{NDAwayZigZagTillOrigin_NumberOfTurnsTheorem}.
\end{proof}

Next, we design a NonZigZag algorithm in which the robots move in opposite directions with a specially chosen (in the course of the proof) speed.

\begin{algorithm}[H]
\caption{NDAwayMovingInOppDirection}\label{NDAwayMovingInOppDirectionAlgorithm}
\begin{algorithmic}[1]
\State $R_{1}$ moves in one direction with speed $u= \frac{3v+1}{3+v}$.
\State $R_{2}$ moves in the other direction with speed $u= \frac{3v+1}{3+v}$.
\If {$R_{1}$ reaches the target}
    \State It switches direction until it catches $R_{2}$.
    \State Both robots switch direction to meet the target.
\Else
    \If {$R_{2}$ reaches the target}
        \State It switches direction until it catches $R_{1}$.
        \State Both robots switch direction to meet the target.
    \EndIf
\EndIf
\end{algorithmic}
\end{algorithm}

We prove the following result:

\begin{theorem}
\label{NDAwayMovingInOppDirectionTheorem}
The competitive ratio of Algorithm~\ref{NDAwayMovingInOppDirectionAlgorithm} is at most 
\[
\frac{(v+3)^2}{(1-v)^2}.
\]
\end{theorem}

\begin{proof}(Theorem~\ref{NDAwayMovingInOppDirectionTheorem})
According to Algorithm~\ref{NDAwayMovingInOppDirectionAlgorithm},
for robot $R_{1}$ to reach the target, it needs time 
\[
\frac{d}{u-v},
\]
where $u= \frac{3v+1}{3+v}$.
At this point, $R_{2}$ will be at a distance $\frac{2du}{u-v}$ from the target. In order for $R_{1}$ to catch $R_{2}$, it needs time 
\[
\frac{2du}{(u-v)(1-u)}.
\]
At this point, both robots will be at a distance of $\frac{2du+2duv}{(u-v)(1-u)}$ from the target. Thus, the time needed for both robots to reach the target is 
\[
\frac{2du+2duv}{(u-v)(1-u)(1-v)}.
\]
We conclude that the total time required for both robots to evacuate is as follows:
\begin{align*}
    &\frac{d}{u-v} + \frac{2du}{(u-v)(1-u)} + \frac{2du+2duv}{(u-v)(1-u)(1-v)} \\
    &= \frac{d(1-v-u+uv)+2du-2duv+2du+2duv}{(u-v)(1-u)(1-v)} \\
    &= \frac{d-dv-du+duv+4du}{(u-v)(1-u)(1-v)}.
\end{align*}

Thus, the competitive ratio is:
\[
CR = f(u) = \frac{1-v+3u+uv}{(u-v)(1-u)}.
\]

The derivative of $f(u)$ is:
\begin{align*}
    f'(u) &= \frac{(3+v)(u-u^2-v+uv) - (1-v+3u+uv)(1-2u+v)}{(u-u^2-v+uv)^2} \\
    &= \frac{3u^2 + vu^2 + 2u - 2uv - 3v - 1}{(u-u^2-v+uv)^2} \\
    &= \frac{(3u+uv-3v-1)(u+1)}{(u-u^2-v+uv)^2}.
\end{align*}

The optimal speed is $u = \frac{3v+1}{3+v}$, and the competitive ratio can be shown to be:
\begin{align*}
    CR &= \frac{1-v+\frac{9v+3}{3+v}+\frac{3v^2+v}{3+v}}{\left(\frac{3v+1}{3+v}-v\right)\left(1-\frac{3v+1}{3+v}\right)} \\
    &= \frac{\frac{3+v-3v-v^2+9v+3+3v^2+v}{3+v}}{\left(\frac{3v+1-3v-v^2}{3+v}\right)\left(\frac{3+v-3v-1}{3+v}\right)} \\
    &= \frac{(v+3)^2}{(1-v)^2}.
\end{align*}
This proves Theorem~\ref{NDAwayMovingInOppDirectionTheorem}.
\end{proof}

\begin{theorem}
\label{NDAwayMovingInOppDirection_NumberOfTurnsTheorem}
The number of turns for Algorithm~\ref{NDAwayMovingInOppDirectionAlgorithm} is $3$, and this is optimal.
\end{theorem}

\begin{proof}(Theorem~\ref{NDAwayMovingInOppDirection_NumberOfTurnsTheorem})
The robot that captures the target needs to turn to inform the other robot. After reaching the other robot, both robots will turn and proceed in the same direction to reach the target. Thus, the total number of turns is $3$.  

Next, we prove that no algorithm can capture the target with fewer than three turns. Recall that both searchers must eventually reach the target. Assume, on the contrary, there is a correct search algorithm that solves the problem with at most two turns. 

First, if one of the searchers does not turn, then the adversary can place the moving target on the side of the origin opposite to the searcher's position, and the searcher will never reach the target—a contradiction. Therefore, at least two turns are required, one by each searcher. 

Second, a searcher should not turn unless it knows the target is in the opposite direction; otherwise, the adversary can place the target farther away in an unvisited area. Since communication is face-to-face, the searcher that finds the target must visit the other searcher, requiring a total of three turns. 

This proves Theorem~\ref{NDAwayMovingInOppDirection_NumberOfTurnsTheorem}.
\end{proof}

\subsection{Target Moving Toward the Origin}

In this section, we consider the case when the target is moving toward the origin. As in Subsection~\ref{sec:Target moving away from the origin}, we first consider a ZigZag algorithm in which the searchers move separately.

\begin{algorithm}[H]
\caption{NDTowardZigZagTillOrigin}\label{NDTowardZigZagTillOriginAlgorithm}
\begin{algorithmic}[1]
 \For{$k \gets 1$ to $\infty$} 
  \State{$R_{1}$ moves left a distance of $a^k$ unless the target is found, then comes back to the origin} 
  \State{$R_{2}$ moves right a distance of $a^k$ unless the target is found, then comes back to the origin}
  \If{Target is found by $R_{1}$}
     \State{Switch direction and move to catch $R_{2}$}
     \State{Both robots move back to catch up with the target}
     \State{Quit}
  \ElsIf{Target is found by $R_{2}$}
     \State{Switch direction and move to catch $R_{1}$}
     \State{Both robots move back to catch up with the target}
     \State{Quit}
  \EndIf
 \EndFor
\end{algorithmic}
\end{algorithm}

We prove the following result:

\begin{theorem}
\label{NDTowardZigZagTillOriginTheorem}
The competitive ratio of Algorithm~\ref{NDTowardZigZagTillOriginAlgorithm}, where $a=\frac{2(1-v)}{1+v}$, is at most $1+\frac{8(1-v)}{(1+v)^2}$.
\end{theorem}
\begin{proof}(Proof of Theorem~\ref{NDTowardZigZagTillOriginTheorem})

If the target is at distance $d$ away from the origin, then eventually either $R_{1}$ or $R_{2}$ will capture the target. If we assume that $R_{1}$ captures the target, then the worst-case scenario occurs when, during iteration $k-1$, $R_{1}$ just misses the target. In other words, when $R_{1}$ reaches $x_{k-1}$, the target would be at distance $x_{k-1} + \epsilon$. Thus, in order to capture the target, $R_{1}$ needs time:

\begin{equation}
     \label{equation5}
    2\sum_{i=0}^{k-1} x_{i} + \frac{d - 2v \left( \sum_{i=0}^{k-1} x_{i} \right)}{1+v}
\end{equation}

At this point, $R_{2}$ is at distance $\frac{d - 2v \left( \sum_{i=0}^{k-1} x_{i} \right)}{1+v}$ to the right of the origin. In order for $R_{1}$ to capture $R_{2}$, it needs time:

\begin{equation}
     \label{equation6}
    x_{k} - \frac{d - 2v \left( \sum_{i=0}^{k-1} x_{i} \right)}{1+v} + \frac{d - 2v \left( \sum_{i=0}^{k-1} x_{i} \right)}{1+v} = x_{k}
\end{equation}

At this time, the target is at a distance of $x_{k} - v x_{k}$ away from both robots. Thus, in order for both robots to capture the target, they need time:

\begin{equation}
     \label{equation7}
    \frac{x_{k} - v x_{k}}{1+v}
\end{equation}

Thus, using equations \ref{equation5}, \ref{equation6}, and \ref{equation7}, the competitive ratio is as follows:

\begin{align*}
    CR &= \frac{2\sum_{i=0}^{k-1} x_{i} + \frac{d - 2v \left( \sum_{i=0}^{k-1} x_{i} \right)}{1+v} + x_{k} + \frac{x_{k} - v x_{k}}{1+v}}{\frac{d}{1+v}} \\
    &\leq \frac{2\sum_{i=0}^{k-1} x_{i} + 2v \sum_{i=0}^{k-1} x_{i} + d - 2v \sum_{i=0}^{k-1} x_{i} + x_{k} + v x_{k} + x_{k} - v x_{k}}{d} \\
    &\leq \frac{2\sum_{i=0}^{k-1} x_{i} + 2x_{k} + d}{d} \\
    &= 1 + \frac{2x_{k+1} - 2}{d(a-1)}
\end{align*}

If we assume that in round $k$, the robot captures the target, then in round $k-1$ we should have the following:

\begin{align}
    &\frac{d - \left( 2a^0 + \dots + 2a^{k-2} \right)v}{1+v} \geq a^{k-1} \nonumber \\
    &\implies d > a^{k-1} + v a^{k-1} + 2v \frac{a^{k-1}-1}{a-1} \nonumber \\
    &\implies a^{k-1} \leq \frac{2v + ad - d}{(a-1)(a+v + \frac{2v}{a-1})} \leq \frac{2v + ad - d}{a + av + v - 1} \label{equation8}
\end{align}

Thus, the competitive ratio is as follows:

\begin{align*}
    CR &\leq 1 + \frac{2a^2 \left( \frac{2v + ad - d}{a + av + v - 1} \right) - 2}{d(a-1)} \\
    &\leq 1 + \frac{4va^2 + 2da^3 - 2da^2 - 2a - 2av - 2v + 2}{d(a-1)(a + av + v - 1)} \\
    &\leq 1 + \frac{4av + 2da^2 - 2 + 2v}{d(a + av + v - 1)} \\
    &\leq 1 + \lim_{d \to \infty} \frac{4av + 2da^2 - 2 + 2v}{d(a + av + v - 1)} \\
    &= 1 + \frac{2a^2}{a + av + v - 1}
\end{align*}

If we set $f(a) = \frac{2a^2}{a + av + v - 1}$, then

\begin{align*}
    f'(a) &= \frac{4a(a + av + v - 1) - 2a^2(1 + v)}{(a + av + v - 1)^2} \\
    &= \frac{4a^2 + 4va^2 + 4av - 4a - 2a^2 - 2va^2}{(a + av + v - 1)^2} \\
    &\leq \frac{2a^2 + 2va^2 + 4av - 4a}{(a + av + v - 1)^2}
\end{align*}

Setting $f'(a) = 0$ gives $2a + 2av + 4v - 4 = 0$. Thus, $a = \frac{2(1-v)}{1+v}$.

Thus, the competitive ratio would be:

\begin{align*}
    CR &\leq 1 + \frac{\frac{8 - 16v + 8v^2}{(1+v)^2}}{\frac{2 - 2v^2 + v + v^2 - 1 - v}{1+v}} \\
    &\leq 1 + \frac{8v^2 - 6v + 8}{(1+v)(1-v^2)} \\
    &= 1 + \frac{8(1-v)}{(1+v)^2}
\end{align*}

This proves Theorem~\ref{NDTowardZigZagTillOriginTheorem}.
\end{proof}
Next, as in Subsection~\ref{sec:Target moving away from the origin}, we consider a NonZigZag algorithm in which the searchers move separately.
The search algorithm is as follows.
\vspace{-0.3cm}
\begin{algorithm}[H]
\caption{NDTowardMovingInOppDirection}\label{alg:tf2f1}
\begin{algorithmic}[1]
\State $R_{1}$ moves in one direction with speed $u= \frac{1-3v}{3-v}$.
\State $R_{2}$ moves in the other direction with speed $u= \frac{1-3v}{3-v}$.
\If {$R_{1}$ reaches the target}
    \State It switches direction until it catches $R_{2}$.
    \State Both robots switch direction to meet the target.
\Else
    \If {$R_{2}$ reaches the target}
        \State It switches direction until it catches $R_{1}$.
        \State Both robots switch direction to meet the target.
    \EndIf
\EndIf
\end{algorithmic}
\end{algorithm}
Note that Algorithm~\ref{alg:tf2f1} requires that $u < 1$ but there is no additional requirement on $v$.

\begin{theorem}
\label{thm:f2f2}
Algorithm~\ref{alg:tf2f1} is correct if $v < \frac{1}{3}$, and its optimal competitive ratio is obtained when $u = \frac{1 - 3v}{3 - v}$; moreover, for that value of $u$ it satisfies
\begin{equation}
\label{eq:thm2}
CR(v) = 1 + \frac{8(1-v)}{(1+v)^2}.
\end{equation}
If $v \geq \frac{1}{3}$, then the waiting algorithm has a competitive ratio of $1 + \frac{1}{v}$.

Moreover, Algorithm~\ref{alg:tf2f1} requires only a total of 3 turns by both robots together for $v \leq \frac{1}{3}$, while the waiting algorithm requires zero turns for $v \geq \frac{1}{3}$.
\end{theorem}

\begin{proof}(Theorem~\ref{thm:f2f2})
The mobile target starts at a distance $d \ge 1$ and moves toward the origin with speed $v$. The algorithm involves three critical meeting points which we specify below:
\begin{enumerate}
\item $T_1$: the time it takes for the first searcher, say $r_1$, to meet the mobile target; let’s call the meeting point $M_1$. Here, both searchers move with speed $u$.
\item $T_2$: the time it takes searcher $r_1$ to catch up with the other searcher $r_2$; let’s call their meeting point $M_2$. Here, $r_1$ moves with speed 1 but $r_2$ moves with speed $u$.
\item $T_3$: the time it takes for the two searchers, $r_1$ and $r_2$, moving together with speed 1, to meet the moving target; let’s call their meeting point $M_3$. Here, both searchers move with speed 1.
\end{enumerate}

In the sequel, we indicate how to calculate the times $T_i$ for $i = 1, 2, 3$. 

{\em Calculating $T_1$:} One of the two searchers, say $r_1$, meets the target first in time
\begin{equation}
\label{eq:ttime1}
T_1 = \frac{d}{u+v}.
\end{equation}
The meeting point with one of the searchers (say $r_1$) is at a point at distance $\frac{du}{u+v}$ from the origin. Since $r_2$ moves with speed $u$, when searcher $r_1$ catches up to the target, searcher $r_2$ will be on the other side of the origin and at distance $\frac{du}{u+v}$ from it. 

{\em Calculating $T_2$:} The distance between the two searchers when $r_1$ meets the target is equal to
$$
\frac{du}{u+v} + \frac{du}{u+v} = \frac{d(2u)}{u+v}.
$$
Now searcher $r_1$ moves with speed 1 and catches up to searcher $r_2$ in additional time
\begin{equation}
\label{eq:ttime2}
T_2 = \frac{ \frac{d(2u)}{u+v} }{1-u} = \frac{d(2u)}{(u+v)(1-u)}.
\end{equation}

{\em Calculating $T_3$:}
The searchers have moved away from the origin. Moreover, the mobile target has been displaced an additional distance $T_2 v$ (towards, and maybe past, the origin), while searcher $r_2$ has been displaced an additional distance $T_2 u$ (away from the origin). Therefore, the distance between the two searchers (who are now together) and the target will be equal to
\begin{align*}
\frac{d(2u)}{u+v} + T_2 (u-v) 
&= \frac{d(2u)}{u+v} + \frac{d(2u)}{(u+v)(1-u)} (u-v) \\ 
&= \frac{d(2u)(1-u)}{(u+v)(1-u)} + \frac{d(2u)(u-v)}{(u+v)(1-u)} \\ 
&= \frac{d(2u)(1-v)}{(u+v)(1-u)}.
\end{align*}

Since the target is moving with speed $v$ and the searchers with speed 1, the time it takes for the two robots to catch the target satisfies 
\begin{equation}
\label{eq:ttime3}
T_3 =  \frac{\frac{d(2u)(1-v)}{(u+v)(1-u)}}{1+v} = \frac{d(2u)(1-v)}{(u+v)(1-u)(1+v)}.
\end{equation}

Using Equations~\eqref{eq:ttime1},~\eqref{eq:ttime2}, and~\eqref{eq:ttime3}, we conclude that the competitive ratio $CR(u,v)$ of the algorithm satisfies
\begin{align*}
CR(u,v) 
&= \frac{T_1 + T_2 + T_3}{\frac{d}{1+v}} \\
&= \frac{\frac{d}{u+v} + \frac{d(2u)}{(u+v)(1-u)} +  \frac{d(2u)(1-v)}{(u+v)(1-u)(1+v)}}{\frac{d}{1+v}} \\
&= \frac{1+v}{u+v} + \frac{(2u)(1+v)}{(u+v)(1-u)} + \frac{(2u)(1-v)}{(u+v)(1-u)} \\
&= 1 + \frac{(1+u)^2}{(1-u)(u+v)}.
\end{align*}

To compute the minimum of $CR(u,v)$ as a function of $u$, we set $\frac{\partial CR(u,v)}{\partial u} = 0$ and solve for $u$ to obtain the equation
$$
u = \frac{1 - 3v}{3 - v}.
$$
If we plug in $u = \frac{1 - 3v}{3 - v}$ into the formula for $CR(u,v)$ above, we obtain
$$
CR(v) := CR \left(\frac{1 - 3v}{3 - v}, v \right)
= 1 + \frac{8(1-v)}{(1+v)^2}.
$$

Given that $0 \leq u < 1$, we note that the value $u = \frac{1 - 3v}{3 - v}$ does not make sense when $v \geq \frac{1}{3}$. In this case, we can employ the waiting algorithm in which both robots wait at the origin for the moving target to arrive. The competitive ratio of the waiting algorithm is $1 + \frac{1}{v}$.

The proof of the assertion of the theorem regarding the number of turns is the same as Theorem~\ref{NDAwayMovingInOppDirection_NumberOfTurnsTheorem}.
\end{proof}

\section{No Speed Model}
\label{sec:NSM}

\subsection{Target Moving Away from the Origin}

For this model, none of the robots know the speed of the target, so each robot will attempt to guess the speed. Let us consider a monotone increasing sequence of non-negative integers $\{f_{i}: i \geq 0\}$. The idea is to guess the speed of the target. We will use the guess \( v_{i} = 1 - 2^{-f_{i}} \). Initially, both robots \( R_1 \) and \( R_2 \) are situated at the origin. The robots move in opposite directions to find the target by guessing its speed. In each iteration, both robots use speed \( u_i > v_i \) with the assumption that the target’s speed is \( v_i \). Each robot moves a necessary distance such that if the target’s speed is less than or equal to \( v_i \), it will be caught in iteration \( i \). If the target is not caught during iteration \( i \), in iteration \( i+1 \), each robot continues in the same direction and increases its speed to \( u_{i+1} \). The robots will continue doing this until one of them catches the target, after which it switches its direction to catch up with the other robot. Both robots will then proceed to the target moving with unit speed.

\begin{algorithm}[H]
\caption{NSAway ($S$ source, $D$ destination)}\label{NSAwayAlgorithm}
Input: Target initial distance \( d \)
\newline Increasing integer sequence \( f_i \) such that \( f_i < f_{i+1} \), \( f_0 = 1 \) and \( t = 0 \)
\begin{algorithmic}[1]
\For{$i \gets 0$ to $\infty$}
  \State $v_i = 1 - 2^{-f_i}$
  \State $R_1$ and $R_2$ move distance $x_i = \frac{d + t v_i}{u_i - v_i}$ in opposite directions with speed $u_i = a_i v_i$, where $a_i = 1 + \frac{1}{2^{2^i}} > 1$, unless the target is found, in which case they continue in the same direction
  \State $t = t + |x_i|$
\EndFor
\If {One of the robots reaches the target}
  \State It switches its direction and moves with unit speed to catch up with the other robot, then both robots proceed to the target with unit speed
\EndIf
\end{algorithmic}
\end{algorithm}

We prove the following results.

\begin{theorem}
\label{NSAwayTheorem}
The competitive ratio of Algorithm~\ref{NSAwayAlgorithm} is bounded from above by
$$
 \frac{5}{2(1-v)^6} + \frac{22 \cdot \log^2 \left( \frac{1}{1-v} \right)}{(1-v)^8}.
$$
\end{theorem}

\begin{proof}(Theorem~\ref{NSAwayTheorem})
Assume that the target is situated at distance \( d \) away from the origin. Let \( d_i \) be the distance from the origin to the target during iteration \( i \). The guess for the speed of the target during iteration \( i \) is \( v_i = 1 - 2^{-f_i} \), and each of the robots moves with speed \( u_i = a_i v_i \), where \( a_i > 1 \). We now have the following equations:

\begin{align*}
    d_i &= d + v_i \sum_{j=0}^{i-1} x_j \\
    x_i &= \frac{d + v_i \sum_{j=0}^{i-1} x_j}{u_i - v_i}.
\end{align*}

Rearranging the above equations, we get:

\begin{align*}
    \sum_{j=0}^{i-1} x_j &= \frac{x_i (u_i - v_i) - d}{v_i} \\
    \implies \sum_{j=0}^{i} x_j &= \frac{x_{i+1}(u_{i+1} - v_{i+1}) - d}{v_{i+1}}.
\end{align*}

After simplification:

\begin{align}
    x_i &= \frac{x_{i+1} (u_{i+1} - v_{i+1} - d)}{v_{i+1}} - \frac{x_i (u_i - v_i - d)}{v_i} \nonumber \\
    \frac{x_{i+1} (u_{i+1} - v_{i+1})}{v_{i+1}} &= x_i + \frac{x_i (u_i - v_i) - d}{v_i} + \frac{d}{v_{i+1}} \nonumber \\
    x_{i+1} &\leq \frac{x_i a_i v_{i+1}}{u_{i+1} - v_{i+1}} \label{NSAwayEquation1}
\end{align}

Now, considering \( a_{i+1} = 1 + \frac{1}{2^{i+1}} \), we have:

\begin{align*}
    \frac{1}{u_{i+1} - v_{i+1}} &= \frac{1}{v_{i+1}(a_{i+1} - 1)} = \frac{2^{2^{i+1}}}{v_{i+1}} \leq 2 \cdot 2^{2^{i+1}}.
\end{align*}

Based on equation~\eqref{NSAwayEquation1}, and considering \( f_i = 2^i \), we get:

\begin{align*}
    x_{i+1} &\leq 2^{f_{i+1}} \cdot 4 \cdot x_i \\
    &\leq 2^{\sum_{j=0}^{i+1} f_j} \cdot 4^{i+1}.
\end{align*}

If the target is captured at iteration \( i \) by one of the robots, then the competitive ratio \( CR \) becomes as follows:

\begin{align*}
    CR &= \frac{\sum_{j=0}^{i-1} x_j + \frac{d + v \sum_{j=0}^{i-1} x_j}{u_i - v} + \frac{2 u_i \sum_{j=0}^{i-1} x_j + 2 u_i \left( \frac{d + v \sum_{j=0}^{i-1} x_j}{u_i - v} \right)}{1 - u_i}}{\frac{d}{1-v}} \\
    &+ \frac{\frac{2u_i \sum_{j=0}^{i-1} x_j + 2 u_i \left( \frac{d + v \sum_{j=0}^{i-1} x_j}{u_i - v} \right) + 2u_i v \sum_{j=0}^{i-1} x_j + 2 v u_i \left( \frac{d + v \sum_{j=0}^{i-1} x_j}{u_i - v} \right)}{(1 - u_i)(1 - v)}}{\frac{d}{1 - v}} \\
     &\leq \frac{u_i \sum_{j=0}^{i-1} x_j + d - u_i v \sum_{j=0}^{i-1} x_j - v d - (u_i)^2 \sum_{j=0}^{i-1} x_j - u_i d}{d (1 - u_i)(u_i - v)} \\
    &+ \frac{u_i^2 v \sum_{j=0}^{i-1} x_j +d u_i v + 4 u_i^{2} \sum_{j=0}^{i-1} x_j + 4 u_i d}{d (1 - u_i)(u_i - v)}.
\end{align*}
Since the target is detected at iteration $i$, we have $2^{2^{i-1}} < \frac{1}{1-v}$, and since $2^{i-1} \leq \log\left(\frac{1}{1-v}\right)$, we get 
\[
4^{i+1} \leq 16 \log^2 \left(\frac{1}{1-v}\right).
\]
Considering $a_i = 1 + \frac{1}{2^{2^i}}$, we have the following:
\begin{align*}
    u_i - v &> u_i - v_i = \frac{v_i}{2^{2^i}} \geq \frac{(1-v)^2}{2}.
\end{align*}
And
\begin{align*}
    1 - u_i &= 1 - \frac{v_i}{2^{2^i}} - v_i \\
    &= 1 - v_i \left(1 + \frac{1}{2^{2^i}}\right) \\
    &= 1 - \left(1 - 2^{-2^i}\right) \left(1 + \frac{1}{2^{2^i}}\right) \\
    &= 1 - 1 - \frac{1}{2^{2^i}} + 2^{-2^i} + \frac{1}{2^{2^{i+1}}} \\
    &= \frac{1}{2^{2^{i+1}}} \geq (1-v)^4.
\end{align*}
Next, we have
\begin{align*}
    \sum_{j=0}^{i-1} x_j &\leq \frac{x_i (u_i - v_i)}{v_i} \\
    &\leq 2 (u_i - v_i)  2^{\sum_{j=0}^{i} f_j} \cdot 4^{i} \\
    &\leq 2\frac{v_{i}}{2^{2^{i}}} \cdot 2^{2^{i+1}} \cdot 4^{i} \\
    &\leq 8\cdot \left(\frac{1}{1-v}\right)^2 \cdot \log^2 \left( \frac{1}{1-v}\right).
\end{align*}
The competitive ratio becomes:

\[
    CR \leq \frac{5}{2(1-v)^6} + \frac{22 \cdot \log^2 \left( \frac{1}{1-v} \right)}{(1-v)^8}.
\]

Thus, Theorem~\ref{NSAwayTheorem} is proven.

\end{proof}

\begin{theorem}
\label{NSAway_NumberOfTurnsTheorem}
The number of turns for Algorithm~\ref{NSAwayAlgorithm} is at most 3 and this is optimal.
\end{theorem}

\begin{proof}(Theorem~\ref{NSAway_NumberOfTurnsTheorem})
The robot that catches up with the target needs to turn to inform the other robot. After reaching the other robot, both robots will turn and proceed in the same direction to reach the target. Thus, the total number of turns will be 3. The optimality is established exactly as in Theorem~\ref{NDAwayMovingInOppDirection_NumberOfTurnsTheorem}.
\end{proof}

\subsection{Target Moving Toward the Origin}

The worst-case scenario occurs when the target's speed is very small. In such cases, if the robots wait at the origin, the competitive ratio can become arbitrarily large. To mitigate this, it is reasonable for both robots to move in one direction for a distance $d$. If the target is not found within this range, the robots then switch direction and continue moving until they catch the target.

\begin{algorithm}[H]
\caption{NSToward ($S$: Source, $D$: Destination)}\label{NSTowardAlgorithm}
\begin{algorithmic}[1]
\State $R_{1}$ and $R_{2}$ choose a direction and move a distance $d$.
\If {the target is not found}
\State $R_{1}$ and $R_{2}$ reverse direction and continue moving until they encounter the target.
\EndIf
\end{algorithmic}
\end{algorithm}

We now prove the following results.

\begin{theorem}
\label{NSTowardTheorem}
The competitive ratio of Algorithm~\ref{NSTowardAlgorithm} is upper bounded by $3$, and this bound is optimal.
\end{theorem}
\begin{proof}[Proof of Theorem~\ref{NSTowardTheorem}]
In the worst-case scenario, if the target's speed is very small, the robots may not catch up to the target immediately. In this case, the robots will move a distance $d$ in one direction. At this point, they will be away from the target by a distance of $2d - dv$, where $v$ is the speed of the target. 

The competitive ratio (CR) can then be calculated as follows:
\begin{align*}
    CR &= \frac{\text{Distance traveled by the robots}}{\text{Distance traveled by the target}} \\
       &= \frac{d + \frac{2d - dv}{1+v}}{\frac{d}{1+v}} \\
       &= 3.
\end{align*}

Hence, the competitive ratio is $3$, which proves Theorem~\ref{NSTowardTheorem}.
\end{proof}

\begin{theorem}
\label{NSTowardTheorem1}
The competitive ratio of the NoSpeedToward model is lower bounded by $3$, and this bound is also optimal.
\end{theorem}
\begin{proof}[Proof of Theorem~\ref{NSTowardTheorem1}]
To prove the lower bound, both points $d$ and $-d$ must be visited by the robots. If only one of these two points, say $-d$, is visited while the other point $d$ is not, the adversary could place the target at $d$. In this case, the robots would fail to catch the target efficiently.

Similarly, if the target is placed at $-d$ and one of the robots visits $d$ first, then by the time the robot reaches $d$, the target would have moved at least $dv$. Following this argument, the competitive ratio would be calculated as follows:
\begin{align*}
    CR &= \frac{\text{Distance traveled by the robots}}{\text{Distance traveled by the target}} \\
       &= \frac{d + \frac{2d - dv}{1+v}}{\frac{d}{1+v}} \\
       &= 3.
\end{align*}

Thus, the competitive ratio is at least $3$, proving the lower bound and Theorem~\ref{NSTowardTheorem1}.
\end{proof}

\section{No Knowledge Model}
\label{sec:NKM}

In this section, we consider the No Knowledge model and distinguish the cases where the target is moving away or toward the origin. 

\subsection{Target Moving Away from the Origin}

We assume that the speed and the initial distance of the target from the origin are unknown. Let us consider two monotone increasing sequences of non-negative integers $\{f_{i}:i\geq 0\}$ and $\{g_{i}:i\geq 0\}$. The idea is to try to guess the speed of the target and its initial distance from the origin. Each robot moves in opposite directions. At iteration $i$, each robot moves with speed $u_{i}=a_{i}v_{i}$, where $a_{i} = 1+2^{-2^i}$, guessing that the speed of the target is $v_{i}$ and that the initial distance from the origin is $d_{i}$. If the target is not found by any of the robots, they continue in the same direction and repeat this process in subsequent iterations until the target is found by one of the robots. The robot that finds the target switches its direction and increases its speed to $1$ (the maximum possible) to catch up with the other robot. After they meet, both robots proceed to meet the moving target. The algorithm is as follows:

\begin{algorithm}[H]
\caption{NKAway ($S$: source, $D$: destination)}\label{NKAwayAlgorithm}
\begin{algorithmic}[1]
\State \textbf{Input:} Target initial distance $d$\newline Increasing integer sequences $\{f_{i}\}, \{g_{i}\}$ such that $f_{i} < f_{i+1}$ and $g_{i} < g_{i+1}$, with $f_{0} = 1$, $g_{0} = 0$, and $t = 0$.
\For{$i \gets 0$ to $\infty$} 
  \State $d_{i} = 2^{g_{i}}$
  \State $v_{i} = 1 - 2^{-f_{i}}$
  \State Robots $R_{1}$ and $R_{2}$ move distance $x_{i} = \frac{d_{i} + t v_{i}}{u_{i} - v_{i}}$ in opposite directions with speed $u_{i} = a_{i}v_{i}$, where $a_{i} = 1 + \frac{1}{2^{2^{i}}}$, unless the target is found, in which case they continue in the same direction.
  \State $t = t + \lvert x_{i} \rvert$
\EndFor
\If {one of the robots reaches the target} 
\State It switches its direction to catch up with the other robot, then both robots proceed to the target.
\EndIf
\end{algorithmic}
\end{algorithm}
\begin{theorem}
\label{thm:uppernoknowledge}
The competitive ratio for the NoKnowledgeAway algorithm is bounded from above by 
\begin{align*}
    &12 M^7 + 192 (\log\log M + 3) \cdot M^{10} \cdot \log^{2}M \cdot \left(\frac{1}{d}\right),
\end{align*}
where $M = \max \left(d, \frac{1}{1-v}\right)$.
\end{theorem}
\begin{proof}(Theorem~\ref{thm:uppernoknowledge})
Assume that the target is situated at distance \( d \) away from the origin.
Let \( d_{i} \) be the distance from the origin to the target during iteration \( i \). Assume that the guess for the speed of the target during iteration \( i \) is \( v_{i} = 1 - 2^{-f_{i}} \), and assume that each of the robots moves with speed \( u_{i} = a_{i}v_{i} \), where \( a_{i} > 1 \). Then we have the following:
\begin{align*}
    d_{i} &= 2^{g_{i}} + v_{i}\sum_{j=0}^{i-1} x_{j}, \quad
    x_{i} = \frac{2^{g_{i}} + v_{i}\sum_{j=0}^{i-1} x_{j}}{u_{i} - v_{i}}.
\end{align*}

As a consequence, we have
\begin{align*}
    \sum_{j=0}^{i-1} x_{j} &= \frac{x_{i}(u_{i} - v_{i}) - 2^{g_{i}}}{v_{i}} \\
    &\implies
    \sum_{j=0}^{i} x_{j} = \frac{x_{i+1}(u_{i+1} - v_{i+1}) - 2^{g_{i+1}}}{v_{i+1}}.
\end{align*}

After simplification, we get the following:
\begin{align}
    &x_{i} = \frac{x_{i+1}(u_{i+1} - v_{i+1}) - 2^{g_{i+1}}}{v_{i+1}}
    - \frac{x_{i}(u_{i} - v_{i}) - 2^{g_{i}}}{v_{i}} \nonumber \\
    &\frac{x_{i+1}(u_{i+1} - v_{i+1})}{v_{i+1}} = x_{i} + \frac{x_{i}(u_{i} - v_{i})}{v_{i}}
    - \frac{2^{g_{i}}}{v_{i}} + \frac{2^{g_{i+1}}}{v_{i+1}} \nonumber \\
    &\frac{x_{i+1}(u_{i+1} - v_{i+1})}{v_{i+1}} = \frac{x_{i}u_{i}}{v_{i}}
    + 2^{g_{i+1}}\left(\frac{1}{v_{i+1}} - \frac{2^{g_{i}}}{2^{g_{i+1}}v_{i}}\right) \nonumber \\
    &\frac{x_{i+1}(u_{i+1} - v_{i+1})}{v_{i+1}} \leq \frac{x_{i}u_{i}}{v_{i}} + 2^{g_{i+1}} \nonumber \\
    &x_{i+1} \leq \frac{x_{i}u_{i}v_{i+1}}{v_{i}(u_{i+1} - v_{i+1})}
    + \frac{2^{g_{i+1}}v_{i+1}}{u_{i+1} - v_{i+1}}.
    \label{NKAwayEquation1}
\end{align}

Consider \( a_{i+1} = 1 + \frac{1}{2^{2^{i+1}}} \). We have the following:
\begin{align*}
    \frac{1}{u_{i+1} - v_{i+1}} &= \frac{1}{v_{i+1}(a_{i+1} - 1)} \\
    &= \frac{2^{2^{i+1}}}{v_{i+1}} \leq 2 \cdot 2^{2^{i+1}}.
\end{align*}

Thus, based on Equation~\eqref{NKAwayEquation1}, and considering $f_{i}=2^{2^{i}}$, we have the following:
\begin{align*}
    x_{i+1} &\leq 2^{f_{i+1}} \cdot 4 \cdot x_{i} + 2 \cdot 2^{g_{i+1}} \cdot 2^{f_{i+1}} \\
    &\leq  2\cdot \sum_{j=0}^{i+1} 2^{g_{k} + \sum_{j=k}^{i+1} f_{j}} \cdot 4^{i-k+1}.
\end{align*}

If the target is captured at iteration \( i \) by one of the robots, then the competitive ratio would be as follows:
\begin{align*}
    CR &= \frac{\sum_{j=0}^{i-1} x_{j}
    + \frac{d_{i} + v\sum_{j=0}^{i-1} x_{j}}{u_{i} - v}
    + \frac{2\sum_{j=0}^{i-1} u_{i}x_{j} + 2u_{i}\left(\frac{d_{i} + v\sum_{j=0}^{i-1} x_{j}}{u_{i} - v}\right)}{1-u_{i}}}{\frac{d}{1-v}} \\
    &+ \frac{\frac{2\sum_{j=0}^{i-1} u_{i}x_{j} + 2u_{i}\left(\frac{d_{i} + v\sum_{j=0}^{i-1} x_{j}}{u_{i} - v}\right)
    + 2v\sum_{j=0}^{i-1} u_{i}x_{j} + 2vu_{i}\left(\frac{d_{i} + v\sum_{j=0}^{i-1} x_{j}}{u_{i} - v}\right)}{(1-u_{i})(1-v)}}{\frac{d}{1-v}}\\
     &=\frac{\frac{u_{i}\sum_{j=0}^{i-1} x_{j}+d_{i}}{u_{i}-v}+\frac{2u_{i}\sum_{j=0}^{i-1} u_{i}x_{j}-2v\sum_{j=0}^{i-1} u_{i}x_{j}+2u_{i}d_{i}+2vu_{i}\sum_{j=0}^{i-1} x_{j}}{(1-u_{i})(1-v)(u_{i}-v)}}{\frac{d}{1-v}}\\
    &+\frac{\frac{2u_{i}\sum_{j=0}^{i-1} u_{i}x_{j}-2v\sum_{j=0}^{i-1} u_{i}x_{j}+2u_{i}d_{i}+2u_{i}v\sum_{j=0}^{i-1} x_{j}+2vu_{i}^{2}\sum_{j=0}^{i-1} x_{j}+2vu_{i}d_{i}}{(1-u_{i})(1-v)(u_{i}-v)}}{\frac{d}{1-v}}\\
    &\leq \frac{u_{i}\sum_{j=0}^{i-1} x_{j}+d_{i}-u_{i}v\sum_{j=0}^{i-1} x_{j}-vd_{i}-(u_{i})^{2}\sum_{j=0}^{i-1} x_{j}-u_{i}d_{i}+u_{i}^{2}v\sum_{j=0}^{i-1} x_{j}}{d(1-u_{i})(u_{i}-v)}\\
    &+\frac{d_{i}u_{i}v+2u_{i}\sum_{j=0}^{i-1} u_{i}x_{j}-2v\sum_{j=0}^{i-1} u_{i}x_{j}+2u_{i}d_{i}+2vu_{i}\sum_{j=0}^{i-1} x_{j}}{d(1-u_{i})(u_{i}-v)}\\
    &+\frac{2u_{i}\sum_{j=0}^{i-1} u_{i}x_{j}-2v\sum_{j=0}^{i-1} u_{i}x_{j}+2u_{i}d_{i}+2u_{i}v\sum_{j=0}^{i-1} x_{j}+2vu_{i}^{2}\sum_{j=0}^{i-1} x_{j}+2vu_{i}d_{i}}{d(1-u_{i})(u_{i}-v)}\\
    &\leq \frac{u_{i}\sum_{j=0}^{i-1} x_{j}+d_{i}-d_{i}v+3u_{i}d_{i}-u_{i}^{2}\sum_{j=0}^{i-1} x_{j}+3u_{i}^{2}v\sum_{j=0}^{i-1} x_{j}+3d_{i}u_{i}v}{d(1-u_{i})(u_{i}-v)}\\
    &+\frac{4u_{i}\sum_{j=0}^{i-1} u_{i}x_{j}-4v\sum_{j=0}^{i-1} u_{i}x_{j}-u_{i}v\sum_{j=0}^{i-1} x_{j}+4u_{i}v\sum_{j=0}^{i-1} x_{j}}{d(1-u_{i})(u_{i}-v)} \\
    &\leq \frac{6d_{i}}{d(1-u_{i})(u_{i}-v)}+\frac{u_{i}(1-v)\sum_{j=0}^{i-1} x_{j}}{d(1-u_{i})(u_{i}-v)}+\frac{6\sum_{j=0}^{i-1} u_{i}x_{j}}{d(1-u_{i})(u_{i}-v)}\\
    &\leq \frac{6d_{i-1}}{(1-u_{i})(u_{i}-v)}+\frac{(1-v)\sum_{j=0}^{i-1} x_{j}}{d(1-u_{i})(u_{i}-v)}+\frac{6\sum_{j=0}^{i-1} x_{j}}{d(1-u_{i})(u_{i}-v)}.
\end{align*}
Since the target is detected at iteration $i$, we have $2^{2^{i-1}} < \frac{1}{1-v}$, $d_{i-1} \leq 2^{g_{i-1}} \leq d$, and since $2^{i-1} \leq \log\left(\frac{1}{1-v}\right)$, we get 
\[
4^{i+1} \leq 16 \log^2 \left(\frac{1}{1-v}\right).
\]
Considering $a_i = 1 + \frac{1}{2^{2^i}}$ and $g_{i}=2^{i}$, we have the following:
\begin{align*}
    u_i - v &> u_i - v_i = \frac{v_i}{2^{2^i}} \geq \frac{(1-v)^2}{2}.
\end{align*}
And
\begin{align*}
    1 - u_i &= 1 - \frac{v_i}{2^{2^i}} - v_i \\
    &= 1 - v_i \left(1 + \frac{1}{2^{2^i}}\right) \\
    &= 1 - \left(1 - 2^{-2^i}\right) \left(1 + \frac{1}{2^{2^i}}\right) \\
    &= 1 - 1 - \frac{1}{2^{2^i}} + 2^{-2^i} + \frac{1}{2^{2^{i+1}}} \\
    &= \frac{1}{2^{2^{i+1}}} \geq (1-v)^4.
\end{align*}
Next, we have
\begin{align*}
    \sum_{j=0}^{i-1} x_j &\leq \frac{x_i (u_i - v_i)}{v_i} \\
    &\leq 4 (u_i - v_i) \sum_{k=0}^{i} 2^{g_k + \sum_{j=k}^{i} f_j} \cdot 4^{i-k} \\
    &\leq 4 (u_i - v_i) \cdot (i+1) \cdot 2^{g_i} \cdot 2^{\sum_{j=0}^{i} f_j} \cdot 4^{i} \\
    &\leq 4\frac{v_{i}}{2^{2^{i}}} \cdot (i+1) \cdot d^2 \cdot 2^{2^{i+1}} \cdot 4^{i} \\
    &\leq 4 (i+1) \cdot \max\left(d, \frac{1}{1-v}\right)^4 \cdot 4^{i} \\
    &\leq 4 ((i-1) + 3) \cdot \max\left(d, \frac{1}{1-v}\right)^4 \cdot 4^{i} \\
    &\leq 16  \cdot \left(\log \log \max\left(d, \frac{1}{1-v}\right) + 3 \right) \\
    &\quad \cdot \max\left(d, \frac{1}{1-v}\right)^4 \cdot \log^2 \left(\max\left(d, \frac{1}{1-v}\right)\right).
\end{align*}
\vspace{.4cm}
We conclude that the competitive ratio becomes as follows:
\begin{align*}
    CR &\leq \frac{6d_{i-1}}{(1-u_i)(u_i - v)} + \frac{(1-v)\sum_{j=0}^{i-1} x_j}{d(1-u_i)(u_i - v)} + \frac{6\sum_{j=0}^{i-1} x_j}{d(1-u_i)(u_i - v)} \\
    &\leq12 \max\left(d, \frac{1}{1-v}\right)^7 + 32 \left(\log \log \max\left(d, \frac{1}{1-v}\right) + 3\right) \\
    &\quad \cdot \max\left(d, \frac{1}{1-v}\right)^{10} \cdot \frac{(1-v)}{d} \cdot \log^2\left(\max\left(d, \frac{1}{1-v}\right)\right) \\
    &+ 192 \left(\log \log \max\left(d, \frac{1}{1-v}\right) + 3\right) \cdot \max\left(d, \frac{1}{(1-v)^2}\right)^{10} \\
    &\quad \cdot \frac{1}{d} \cdot \log^2\left(\max\left(d, \frac{1}{1-v}\right)\right).
\end{align*}
The result follows by simplifying the expression above after setting $M = \max\left(d, \frac{1}{1-v}\right)$. This proves Theorem~\ref{thm:uppernoknowledge}.
\end{proof}
\begin{theorem}
\label{NKAway_NumberOfTurnsTheorem}
The number of turns for Algorithm~\ref{NKAwayAlgorithm} is at most $3$.
\end{theorem}

\begin{proof}[Proof of Theorem~\ref{NKAway_NumberOfTurnsTheorem}]
The robot that catches up to the target needs to turn to inform the other robot. After reaching the other robot, both robots will turn and proceed in the same direction to reach the target. Thus, the total number of turns is at most $3$. This proves Theorem~\ref{NKAway_NumberOfTurnsTheorem}.
\end{proof}
\subsection{Target moving toward the origin}
If both robots wait at the origin, then the competitive ratio would be at least \( 1 + \frac{1}{v} \) as shown below.

\begin{theorem}
\label{thm:noknowlower}
The competitive ratio of any algorithm in the NoKnowledgeToward model is bounded from below by 
\[
1 + \frac{1}{v}.
\]
\end{theorem}

\begin{proof} (Theorem~\ref{thm:noknowlower})
    The optimal competitive ratio follows from having both robots wait at the origin. If we assume that one of the two robots catches up to the target at a distance \( d^{\prime} \), then the other robot would be at a distance of \( -d^{\prime} \) away from the origin. Alternatively, if both robots are at distance \( d^{\prime} \), then the adversary would have placed the target at distance \( -d^{\prime} \). In this case, the competitive ratio would be as follows:
    \begin{align*}
        CR &= \frac{\frac{d - d^{\prime}}{v} + \frac{2d^{\prime}}{1 + v}}{\frac{d}{1 + v}} \\
           &= \frac{\frac{d - d^{\prime} + dv - vd^{\prime} + 2d^{\prime}v}{v(1 + v)}}{\frac{d}{1 + v}} \\
           &= 1 + \frac{1}{1 + v} + \frac{d^{\prime}}{d} \\
           &> \frac{1}{1 + v}.
    \end{align*}
    Thus, we conclude that the optimal competitive ratio of NoKnowledgeToward is \( 1 + \frac{1}{v} \). This proves Theorem~\ref{thm:noknowlower}.
\end{proof}

\section{Conclusion}
We considered the problem of two robots capturing an oblivious moving target on an infinite line. Two cases were considered depending on whether the target is moving toward or away from the origin. In each of these two cases, we considered different constraints based on the knowledge about the speed and the initial distance of the target from the origin. Our algorithms are optimal in the number of turns required to achieve the desired competitive ratio. All algorithms were based on the F2F communication model.

It remains an open problem to prove tight bounds for the case when the distance is unknown, and the target is moving away from the origin. It also appears that there is a tradeoff between the competitive ratio and the number of direction changes during the execution of the search algorithm. However, it is an open problem to determine the optimal tradeoffs. 

An interesting and challenging open problem, motivated by \cite{isaacCzyzowiczGKKNOS16,PODC16}, concerns the \( n \)-searcher problem for capturing a moving target when at most \( f \) of the searchers may be byzantine or crash faulty, respectively.

\bibliographystyle{abbrv}
\bibliography{refs}

\begin{thebibliography}{10}

\bibitem{ahlswede1987search}
R.~Ahlswede and I.~Wegener.
\newblock {\em Search problems}.
\newblock Wiley-Interscience, 1987.

\bibitem{gal_search_games}
S.~Alpern and S.~Gal.
\newblock {\em The theory of search games and rendezvous}, volume~55 of {\em International series in operations research and management science}.
\newblock Kluwer, 2003.

\bibitem{angelopoulos2017infinite}
S.~Angelopoulos, D.~Ars{\'e}nio, and C.~D{\"u}rr.
\newblock Infinite linear programming and online searching with turn cost.
\newblock {\em Theoretical Computer Science}, 670:11--22, 2017.

\bibitem{BCR93}
R.~A. Baeza{-}Yates, J.~C. Culberson, and G.~J.~E. Rawlins.
\newblock Searching in the plane.
\newblock {\em Inf. Comput.}, 106(2):234--252, 1993.

\bibitem{bampas2019linear}
E.~Bampas, J.~Czyzowicz, L.~Gasieniec, D.~Ilcinkas, R.~Klasing, T.~Kociumaka, and D.~Pajak.
\newblock Linear search by a pair of distinct-speed robots.
\newblock {\em Algorithmica}, 81(1):317--342, 2019.

\bibitem{beck1964linear}
A.~Beck.
\newblock On the linear search problem.
\newblock {\em Israel Journal of Mathematics}, 2(4):221--228, 1964.

\bibitem{bellman1963optimal}
R.~Bellman.
\newblock An optimal search.
\newblock {\em SIAM Review}, 5(3):274--274, 1963.

\bibitem{anthony2011game}
A.~Bonato and R.~Nowakowski.
\newblock {\em The game of cops and robbers on graphs}.
\newblock American Mathematical Soc., 2011.

\bibitem{brandt2017collaboration}
S.~Brandt, F.~Laufenberg, Y.~Lv, D.~Stolz, and R.~Wattenhofer.
\newblock Collaboration without communication: Evacuating two robots from a disk.
\newblock In {\em International Conference on Algorithms and Complexity}, pages 104--115. Springer, 2017.

\bibitem{chrobak2015group}
M.~Chrobak, L.~G{\k{a}}sieniec, T.~Gorry, and R.~Martin.
\newblock Group search on the line.
\newblock In {\em International conference on current trends in theory and practice of informatics}, pages 164--176. Springer, 2015.

\bibitem{chuangpishit2018evacuating}
H.~Chuangpishit, S.~Mehrabi, L.~Narayanan, and J.~Opatrny.
\newblock Evacuating an equilateral triangle in the face-to-face model.
\newblock In {\em 21st International Conference on Principles of Distributed Systems (OPODIS 2017)}. Schloss Dagstuhl-Leibniz-Zentrum fuer Informatik, 2018.

\bibitem{coleman2024lineiwoca}
J.~Coleman, D.~Ivanov, E.~Kranakis, D.~Krizanc, and O.~Morales-Ponce.
\newblock Linear search for an escaping target with unknown speed.
\newblock {\em Proceedings of IWOCA}, 2024.

\bibitem{coleman2022line}
J.~Coleman, E.~Kranakis, D.~Krizanc, and O.~Morales-Ponce.
\newblock Line search for an oblivious moving target.
\newblock {\em Proceedings of OPODIS}, 2022.

\bibitem{czyzowicz2018evacuating}
J.~Czyzowicz, S.~Dobrev, K.~Georgiou, E.~Kranakis, and F.~MacQuarrie.
\newblock Evacuating two robots from multiple unknown exits in a circle.
\newblock {\em Theoretical Computer Science}, 709:20--30, 2018.

\bibitem{group_search}
J.~Czyzowicz, K.~Georgiou, and E.~Kranakis.
\newblock Group search and evacuation.
\newblock In P.~Flocchini, G.~Prencipe, and N.~Santoro, editors, {\em Distributed Computing by Mobile Entities: Current Research in Moving and Computing}, Cham, 2019. Springer International Publishing.

\bibitem{isaacCzyzowiczGKKNOS16}
J.~Czyzowicz, K.~Georgiou, E.~Kranakis, D.~Krizanc, L.~Narayanan, J.~Opatrny, and S.~Shende.
\newblock Search on a line by byzantine robots.
\newblock In {\em {ISAAC}}, pages 27:1--27:12, 2016.

\bibitem{PODC16}
J.~Czyzowicz, E.~Kranakis, D.~Krizanc, L.~Narayanan, and J.~Opatrny.
\newblock Search on a line with faulty robots.
\newblock {\em Distributed Computing}, 32(6):493--504, 2019.

\bibitem{demaine2006online}
E.~D. Demaine, S.~P. Fekete, and S.~Gal.
\newblock Online searching with turn cost.
\newblock {\em Theoretical computer science}, 361(2-3):342--355, 2006.

\bibitem{jawhar2021robot}
K.~Jawhar and E.~Kranakis.
\newblock Robot evacuation on a line assisted by a bike.
\newblock {\em Information}, 12(1):28, 2021.

\bibitem{mccabe1974}
B.~J. McCabe.
\newblock Searching for a one-dimensional random walker.
\newblock {\em J. Applied Probability}, pages 86--93, 1974.

\bibitem{chases_escapes}
P.~J. Nahin.
\newblock {\em Chases and Escapes, The Mathematics of Pursuit and Evasion}.
\newblock Princeton University Press, Princeton, 2012.

\end{thebibliography}
\end{document}